\documentclass[11pt,twocolumn]{article}
\usepackage{bagrow}
\usepackage{mathtools}
\usepackage{enumitem}
\usepackage{microtype}

\newcommand{\norm}[1]{\left\lVert#1\right\rVert}
\DeclareMathOperator{\SiLU}{SiLU}
\newcommand{\R}{\mathbb{R}}
\newcommand{\E}{\mathbb{E}}
\newcommand{\wb}[1][]{\alpha_{#1}}
\newcommand{\ws}[1][]{\beta_{#1}}
\newcommand{\Rcomp}{\mathcal{R}}

\usepackage{comment}
\excludecomment{dead}
\usepackage{amsthm}
\newtheorem{theorem}{Theorem}

\title{KANs need curvature: penalties for compositional smoothness}
\author[1,2,*]{James Bagrow}
\affil[1]{Mathematics \& Statistics, University of Vermont, Burlington, VT, United States}
\affil[2]{Vermont Complex Systems Center, University of Vermont, Burlington, VT, United States}
\affil[*]{\corrauthinfo{james.bagrow@uvm.edu}{bagrow.com}}
\date{}

\begin{document}

\makeatletter
\twocolumn[%
  \maketitle
  \begin{@twocolumnfalse}
  \begin{abstract}\small
Kolmogorov--Arnold networks (KANs) offer a potent combination of accuracy and interpretability, thanks to their compositions of learnable univariate activation functions.
However, the activations of well-fitting KANs tend to exhibit pathologically high-curvature oscillations, making them difficult to interpret, and standard regularization penalties do not prevent this.
Here we derive a basis-agnostic curvature penalty and show that penalized models can maintain accuracy while achieving substantially smoother activations. 
Accounting for how function composition shapes curvature, we prove an upper bound on the full model's curvature relative to the curvature penalty, and use this to motivate richer forms of penalties.
Scientific machine learning is increasingly bottlenecked by the trade-off between accuracy and interpretability. 
Results such as ours that improve interpretability without sacrificing accuracy will further strengthen KANs as a practical tool for both prediction and insight.
  \end{abstract}
  \vspace{1em}
  \end{@twocolumnfalse}
]
\makeatother

\section{Introduction}

Kolmogorov--Arnold networks (KANs)~\cite{liu2024kan} are growing in popularity as an alternative to traditional neural networks.
KANs replace fixed nonlinearities with learnable univariate activations placed on each edge.
This makes them effective in two ways~\cite{somvanshi2025survey}: they are highly accurate and more interpretable than standard deep networks.
When both accuracy and interpretability matter, such as in scientific machine learning, KANs are a natural choice~\cite{liu2025kanmeetscience}.

In practice, KANs often fit data well but develop high-curvature, kink-like oscillations in their activation functions~\cite{samadi2024smooth,yu2024can} (Fig.~\ref{fig:anecdote}).
This is problematic because interpretable activations must be readable, that is, smooth.\footnote{``Smooth'' in the sense of low bending energy (roughness), not continuous differentiability.}
As Liu \textit{et al.}\ remark, ``deeper [more than two layers] representations may bring the advantages of smoother activations''~\cite[\S2.3]{liu2024kan}.
Yet KART~\cite{kolmogorov1957representation,arnold1957functions}, the representation theorem underpinning KANs, offers no smoothness guarantee~\cite{vitushkin1964proof}.
Much recent work has explored KAN architectures~\cite{vaca2024kolmogorov,kiamari2024gkan,yang2024kolmogorov,bagrow2025multi,bagrow2025optimized}, training methods~\cite{rigas2024adaptive,rigas2026deeppikan,rigas2026initialization}, and downstream applications~\cite{panahi2025discovery,panczyk2025opening,bagrow2025softly}, but the question of how to enforce smoothness in KAN activations remains open.

\begin{figure*}
  \centering
  \includegraphics[width=0.85\linewidth]{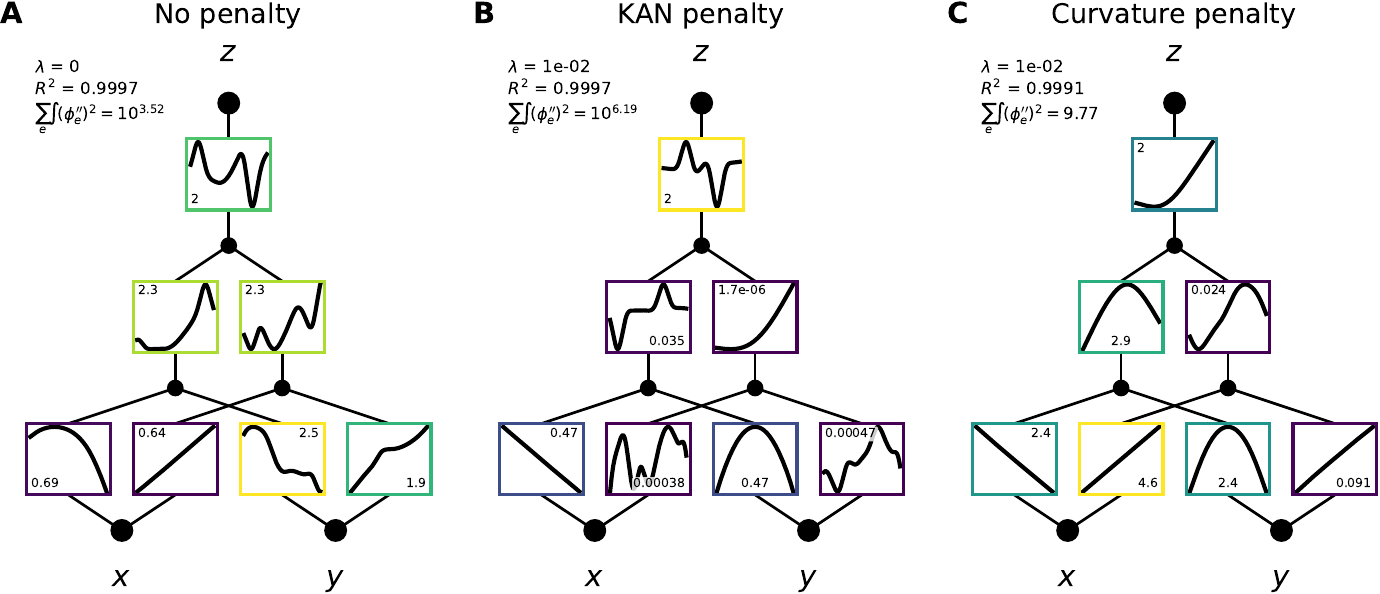}
  \caption{Trained networks on $\sin\left(x+y^2\right)$ over $\left[-2, 2\right]^2$ with grid size $G=10$.
  Despite accurate fits, the unpenalized (\textbf{A}) and KAN-penalized (\textbf{B}) activations exhibit high-curvature, kink-like oscillations, very unlike the true function's components in appearance.
  The curvature-penalized model (\textbf{C}) presents smooth activation functions more closely aligned with the true function's $x$, $y^2$, and $\sin(\cdot)$ components.
  Plot frame color and label indicate the operating range ($\max_z \phi_e - \min_z \phi_e$) of each activation.}
  \label{fig:anecdote}
\end{figure*}

To address this gap, we make the following contributions:
\begin{enumerate}[leftmargin=*]
\item
We show that the standard KAN penalty structurally cannot suppress activation curvature and so cannot guarantee smooth activations (Sec.~\ref{sec:kan-penalty}).

\item
We derive a closed-form edge-wise curvature penalty that is basis-agnostic (Sec.~\ref{sec:activation-curvature}), and demonstrate empirically that it yields substantially smoother KANs at similar accuracy, across function approximation, the Feynman benchmark, and overparameterized regimes (Sec.~\ref{sec:experiments}).

\item
We prove through a compositional analysis that the edge-wise penalty upper-bounds the curvature of the full network (Sec.~\ref{sec:compositional-curvature}), and use this result to motivate richer penalties that draw on more than edge-wise information (Sec.~\ref{sec:richer-penalty}).
\end{enumerate}

Section~\ref{sec:related} positions our work relative to prior P-spline, KAN-curvature, and input-Hessian literature, and Sec.~\ref{sec:discussion} concludes.

\section{KAN's penalty can't enforce smoothness}
\label{sec:kan-penalty}

A KAN of depth $L$ and widths $[n_0, n_1, \ldots, n_L]$ maps an input $x \in \R^{n_0}$ to an output $f(x) \in \R^{n_L}$ via the layer-wise composition
\begin{equation}
\label{eq:kan-forward}
z^{(\ell)}_c \;=\; \sum_{b=1}^{n_{\ell-1}} \phi^{(\ell)}_{cb}\!\left(z^{(\ell-1)}_b\right),
\end{equation}
for $\ell = 1, \ldots, L$ and $c = 1, \ldots, n_\ell$, where $z^{(0)} = x$, $f(x) = z^{(L)}$, and each edge $b \to c$ at layer $\ell$ carries its own learnable univariate activation function $\phi^{(\ell)}_{cb}: \R \to \R$.
The original and most popular KAN form~\cite{liu2024kan} implements activations with a base (residual) function, typically SiLU, and B-splines:
\begin{equation}
\label{eq:phi-def}
\phi(x) \;=\; \wb\, \SiLU(x) \;+\; \ws\, c^\top B(x),
\end{equation}
where $B(x) = (B_1(x), \ldots, B_G(x))^\top$ are the B-spline basis functions on a fixed (uniform, we assume) knot grid of size $G$, $\wb, \ws \in \R$ are scalar weights, and $c \in \R^{G}$ is the learnable spline coefficient vector. 
SiLU is the canonical choice for the base following~\cite{liu2024kan} but any smooth substitute carries through our curvature analysis; non-B-spline bases such as Gaussian-RBF KANs are treated in App.~\ref{sec:beyond-bsplines}.

KANs are trained by minimizing $J(f) + \lambda R(f)$, where $J$ is a data-fitting loss (mean squared error throughout this paper) and $R$ is a penalty with strength $\lambda \ge 0$. Liu \textit{et al.}~\cite{liu2024kan} propose the \textbf{KAN penalty},
\begin{equation}
\label{eq:pykan-reg}
R_\mathrm{KAN}(f) \;=\; \mu_1 \sum_e \big|\phi_e\big|_1 \;+\; \mu_2 \sum_\ell S\!\left(\rho^{(\ell)}\right),
\end{equation}
where the first sum runs over all edges of the network, $|\phi|_1 := \tfrac{1}{N}\sum_{s=1}^{N} |\phi(z^{(s)})|$ is an activation's average magnitude over $N$ training samples, $S(\rho^{(\ell)}) := -\sum_{b, c} \rho^{(\ell)}_{cb} \log \rho^{(\ell)}_{cb}$ is the entropy of the within-layer magnitude distribution, and $\rho^{(\ell)}_{cb} := |\phi^{(\ell)}_{cb}|_1 / \sum_{b',c'} |\phi^{(\ell)}_{c' b'}|_1$ is the normalized magnitude of edge $b \to c$ at layer $\ell$.
The first term shrinks activation magnitudes while the second concentrates each layer's mass onto a few high-magnitude edges.

Figure~\ref{fig:anecdote} shows a KAN trained with no penalty (A) and with the penalty of Eq.~\ref{eq:pykan-reg} (B) on $f(x, y) = \sin(x + y^2)$.
Both have high test accuracy yet many high-curvature, oscillatory activation functions.
Why?

In fact, it is structurally impossible for Eq.~\ref{eq:pykan-reg} to smooth out activations.
To see this, observe that both terms in $R_\mathrm{KAN}$ depend only on the average magnitudes $|\phi^{(\ell)}_{cb}|_1$, which carry no derivative information about $\phi$. 
A wildly oscillating $\phi$ and a smooth $\phi$ with the same average magnitude will incur the same penalty. 
The KAN penalty governs \emph{which} edges carry magnitude, not how it varies across their support.

What is needed instead, and which is well known in the study of splines as \textit{P-splines}~\cite{eilers1996flexible}, is to penalize the curvature of $\phi$ directly, as we discuss in the next section.

\section{Activation function curvature}
\label{sec:activation-curvature}

A natural choice of curvature penalty on a sufficiently smooth function $g$ is its $L^2$ bending energy $\int |\nabla^2 g(x)|_F^2\, dx$, the squared $H^2$ Sobolev seminorm.\footnote{Throughout, ``curvature'' refers to bending energy $\int (\phi'')^2\, dz$, whose integrand $(\phi'')^2$ is the small-slope linearization of squared differential-geometric curvature $(\phi'')^2/(1+\phi'^2)^3$. Bending energy is the standard surrogate elsewhere.}
For a univariate $g$ this becomes $\int (g''(z))^2 \, dz$. 
We compute this term for each KAN edge activation $\phi_e$ on its spline support $\Omega_e \subset \R$, then sum over edges to obtain the curvature penalty $R(f)$, the focus of this paper.
(A penalty of this form has been applied to KANs~\cite{yu2024can}, with crucial differences; see Sec.~\ref{sec:related}.)
This edge-wise penalty is related to the full curvature of $f$ in Sec.~\ref{sec:compositional-curvature}.

Substituting the KAN edge form (Eq.~\ref{eq:phi-def}) and integrating gives
\begin{multline}
\label{eq:phi-double-prime-squared}
\int_{\Omega_e}\!\!\big(\phi_e''(z)\big)^2 \,dz \;=\; \ws[e]^2 \!\!\int\!\! (s_e'')^2\,dz \;+\; \wb[e]^2 \!\!\int\!\! (u'')^2\,dz \\
\;+\; 2\,\ws[e]\,\wb[e] \!\!\int\!\! s_e''\, u''\,dz,
\end{multline}
where $s_e(z) := c_e^\top B(z)$ is the spline component and $u(z) := \SiLU(z)$ is the fixed base. 
We treat the three terms as follows:
\begin{enumerate}[leftmargin=*]%

\item The Eilers--Marx P-spline reduction~\cite{eilers1996flexible} gives $\int_{\Omega_e}\!(s_e''(z))^2\, dz \;\approx\; h_e^{-3}\, \norm{D_2 c_e}^2$ on uniform-knot cubic B-splines, where $D_2$ is the second-difference matrix on the $G$ spline coefficients and $h_e = |\Omega_e|/G$ is the knot spacing. %
We drop the $h_e^{-3}$ prefactor to prevent high-$G$ grids from dominating the base term.

\item  The integral is the constant $K_{\mathrm{silu}} := \int (u'')^2\, dz = (30 + \pi^2)/90 \approx 0.443$; because $u$ is fixed, the resulting $\wb[e]^2 K_{\mathrm{silu}}$ term acts as $L^2$ shrinkage on $\wb[e]$.

\item The cross term is dropped. 
Its magnitude is bounded by the sum of squared terms and is small in practice because $u''$ is localized near $z=0$, but more importantly, keeping it may let the optimizer reduce the penalty by anti-correlating $s_e''$ and $u''$ rather than minimizing curvature.
\end{enumerate}
Combining (1)--(3) and summing over edges, with the scalar $\ws[e]$ folded inside the norm, %
yields our proposed \textbf{curvature penalty}:
\begin{equation}
\label{eq:curv-penalty}
R(f) \;=\; \sum_e \left(\norm{D_2(\ws[e]\, c_e)}^2 + K_{\mathrm{silu}}\, \wb[e]^2\right).
\end{equation}

\paragraph{Remark.}
The proposed penalty has several nice properties:
It works entirely on model coefficients and does not couple to training data.
It has affine functions as its null space.
It admits a clean Bayesian interpretation as a Gaussian prior on second differences.
When applied to B-splines, the term $\norm{D_2 c_e}^2$ is the central object of study in the field of penalized splines (P-splines)~\cite{eilers2010splines,eilers2015twenty}, a rich field with a 30-year history that to the best of our knowledge has gone untouched by the KAN community.

\paragraph{Remark.}
The standard PyKAN package also implements a first-difference penalty on spline coefficients alongside Eq.~\ref{eq:pykan-reg} (disabled by default).
But this also does not work well as a curvature penalty, instead serving as a sparsity prior that concentrates curvature on a small number of knots.
Its null space is constants, not affine functions.
This makes it more tailored to reward piecewise-constant functions than low-curvature ones.

\paragraph{Remark.}
More generally, we can apply our curvature penalty to any fixed basis whose second derivatives are square-integrable: $\int_{\Omega_e}\!(s_e''(z))^2\, dz = c_e^\top M_e\, c_e$ with curvature Gram matrix $(M_e)_{ij} = \int B_i''(z)\, B_j''(z)\, dz$.
We explore a non-B-spline basis for $\phi$ in App.~\ref{sec:beyond-bsplines}.

\paragraph{Remark.}
The curvature penalty works edge-wise. It is not a curvature penalty for the full function $f(x)$ learned by the KAN because it does not capture the composition of activation functions across layers. 
We return to this in Sec.~\ref{sec:compositional-curvature} and show that the edge-wise penalty is an upper bound on the curvature of the full composition.

\section{Lower curvature models with similar accuracy}
\label{sec:experiments}

Here we compare the curvature penalty to no-penalty and KAN-penalty conditions for several tasks commonly addressed by KAN models.
Experimental details are given in Methods (App.~\ref{sec:methods}).

\subsection{KANs for function approximation}

Comparing results across two dimensions, test accuracy (RMSE or $R^2$) and total edge curvature $\sum_e \int (\phi_e''(z))^2\, dz$,
we find that curvature-penalized KANs achieve substantially smoother activation functions at the same or nearly the same fit quality.
For instance, in Fig.~\ref{fig:anecdote}, all KANs fit the target function very well, but the no-penalty and KAN-penalty models do so with solutions that have orders of magnitude more curvature than the curvature-penalized model.
Inspecting the individual $\phi_e$ in the plot, the curvature-penalized KAN most closely resembles the components of the target function, suggesting that the curvature penalty can aid in interpretability and downstream tasks such as symbolic regression.

Figure~\ref{fig:g10-lam-sweep} compares no-penalty and curvature-penalized KANs across a sweep of penalty strengths, for the target $f(x, y) = \exp\left(\sin\left(\pi x\right) + y^2\right)$ studied by Liu \emph{et al.}~\cite{liu2024kan}.
The unpenalized model fits very well, with test RMSE $<10^{-3}$, but the penalized KAN is nearly as accurate over a wide range of $\lambda$ values. 
Even with the relatively low grid size ($G=10$), which itself should act as a smoothness prior, the penalized model achieves total edge curvatures one-third or less that of the unpenalized model.

\begin{figure}
  \centering
  \makeatletter
  \includegraphics[width=\if@twocolumn 0.95\else 0.7\fi\linewidth]{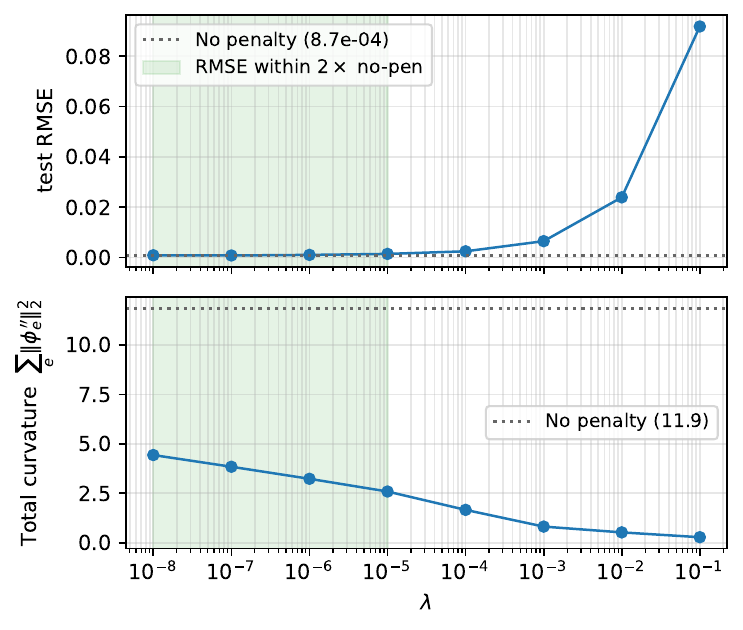}
  \makeatother
  \caption{The penalty selects for a smoother basin among same-fit minima.
  Approximating $f(x, y) = \exp\left(\sin\left(\pi x\right) + y^2\right)$ 
  with architecture $[2,5,1]$ and $G=10$.
  The green shaded region:
  $\lambda$ range where penalized test RMSE is within $2\times$
  the unpenalized baseline (dashed).
  }
  \label{fig:g10-lam-sweep}
\end{figure}

This result---reduced curvature with only a small cost to accuracy---holds broadly. In Table~\ref{tab:feynman} we study the effects of penalties for KANs approximating popular functions taken from the Feynman Equation symbolic regression benchmark.
Again, curvature-penalized KANs are the smoothest models, and are the most accurate (3 of 14 equations) or within no more than a factor of two error of the most accurate model (10 of 14 equations).
\begin{table*}[ht]
  \centering
  \scriptsize
  \setlength{\tabcolsep}{4pt}
  \begin{tabular}{l l ccc c ccc}
    \toprule
    & & \multicolumn{3}{c}{Test RMSE (median)} & & \multicolumn{3}{c}{$\sum_e\!\int(\phi_e'')^2$ (median)} \\
    \cmidrule(lr){3-5} \cmidrule(lr){7-9}
    Feynman eq.\ & Formula & No penalty & KAN & Curvature & & No penalty & KAN & Curvature \\
    \midrule
    I.6.20 & $e^{-\theta^2/(2\sigma^2)} / \sqrt{2\pi\sigma^2}$ & \textbf{0.0015} & \textit{0.0022} & 0.0034 & & $52.7$ & $305$ & $\mathbf{17.7}$ \\
    I.6.20b & $e^{-(\theta-\theta_1)^2/(2\sigma^2)} / \sqrt{2\pi\sigma^2}$ & \textbf{0.0058} & \textit{0.0075} & \textit{0.0070} & & $74.8$ & $107$ & $\mathbf{11.8}$ \\
    I.9.18 & $G m_1 m_2 / \big[(\Delta x)^2 + (\Delta y)^2 + (\Delta z)^2\big]$ & \textbf{0.0022} & 0.0069 & \textit{0.0022} & & $118$ & $1.2\!\cdot\!10^{4}$ & $\mathbf{14.7}$ \\
    I.12.11 & $q\,(E_f + B v \sin\theta)$ & \textit{0.0751} & \textit{0.0720} & \textbf{0.0560} & & $112$ & $99.5$ & $\mathbf{36.6}$ \\
    I.16.6 & $uv / (1 + uv/c^2)$ & \textbf{0.0061} & \textit{0.0075} & \textit{0.0070} & & $179$ & $278$ & $\mathbf{29.0}$ \\
    I.18.4 & $(m_1 r_1 + m_2 r_2) / (m_1 + m_2)$ & \textit{0.0013} & 0.0023 & \textbf{0.0009} & & $60.2$ & $480$ & $\mathbf{29.0}$ \\
    I.26.2 & $\arcsin(n \sin\theta_2)$ & \textbf{0.0094} & \textit{0.0098} & \textit{0.0136} & & $84.7$ & $72.3$ & $\mathbf{19.2}$ \\
    I.29.16 & $\sqrt{x_1^2 + x_2^2 - 2 x_1 x_2 \cos(\theta_1-\theta_2)}$ & \textbf{0.1256} & \textit{0.1527} & \textit{0.1294} & & $390$ & $393$ & $\mathbf{63.1}$ \\
    I.30.3 & $I_0\,\sin^2(n\theta/2) / \sin^2(\theta/2)$ & \textit{0.1787} & \textit{0.1789} & \textbf{0.1763} & & $619$ & $559$ & $\mathbf{153}$ \\
    I.50.26 & $x_1\,[\cos(\omega t) + \alpha\cos^2(\omega t)]$ & \textbf{0.0365} & \textit{0.0367} & \textit{0.0450} & & $194$ & $136$ & $\mathbf{67.9}$ \\
    II.11.27 & $n\alpha\,\epsilon E_f / (1 - n\alpha/3)$ & \textit{0.0085} & \textbf{0.0076} & \textit{0.0084} & & $766$ & $1088$ & $\mathbf{627}$ \\
    II.35.18 & $n_0 / [\,e^{\mu B/k_b T} + e^{-\mu B/k_b T}\,]$ & \textbf{0.0016} & 0.0056 & \textit{0.0018} & & $51.5$ & $2.1\!\cdot\!10^{6}$ & $\mathbf{19.6}$ \\
    III.10.19 & $\mu\,\sqrt{B_x^2 + B_y^2 + B_z^2}$ & \textbf{0.0038} & 0.0081 & \textit{0.0054} & & $2340$ & $2744$ & $\mathbf{258}$ \\
    III.17.37 & $\beta\,(1 + \alpha \cos\theta)$ & \textbf{0.0089} & \textit{0.0095} & \textit{0.0161} & & $95.4$ & $76.3$ & $\mathbf{20.4}$ \\
    \bottomrule
  \end{tabular}
  \caption{Feynman benchmark, $G=10$, $n=5$ seeds, architecture $[d_{\text{in}}, d_{\text{in}}, 1]$. Each penalty is shown at the single $\lambda^*$ that minimizes the geometric mean of test RMSE across all 14 equations: $\lambda^*=0.0001$ for the KAN penalty and $\lambda^*=0.0001$ for the curvature penalty. Winner in each RMSE block bolded per row; non-winners within $2\times$ the winner italicized. Curvature winner bolded per row. The curvature penalty matches or beats the KAN penalty on median RMSE in 9/14 equations and has the lowest total edge curvature in 14/14.}
  \label{tab:feynman}
\end{table*}

\subsection{Curvature penalty stabilizes overparameterized KANs}

When $G$ is large, the KAN activations are over-parameterized and training becomes unstable.
Standard practice has been to use grid extension~\cite{liu2024kan}: start the model at low-$G$ and then progressively increase $G$, each time re-projecting the activations from the old to the new basis using least-squares.
While this can be effective, in part because starting at low $G$ serves as a smoothness inductive bias, ideally one would prefer to spend all of one's training budget on the full resolution model.

Figure \ref{fig:g200-comparison} shows the effects the curvature penalty can have on training stability for over-parameterized KANs.
Compared to the KAN penalty, the curvature penalty achieves lower test RMSE for all values of $\lambda$.
At its optimal $\lambda$, the curvature-penalized KAN is still improving at the end of 3000 Adam epochs, while the KAN penalty model plateaus after $\sim$ 1000 epochs.
While a grid extension run can still achieve lower RMSE~\cite{liu2024kan}, the gap is itself informative: it suggests that much of grid extension's effectiveness comes from the implicit smoothness regularization imposed by its low-$G$ early stages, and that the explicit curvature penalty captures a substantial portion of this benefit at fixed high $G$.

\begin{figure*}[ht]
\begin{center}
\includegraphics[width=0.85\textwidth]{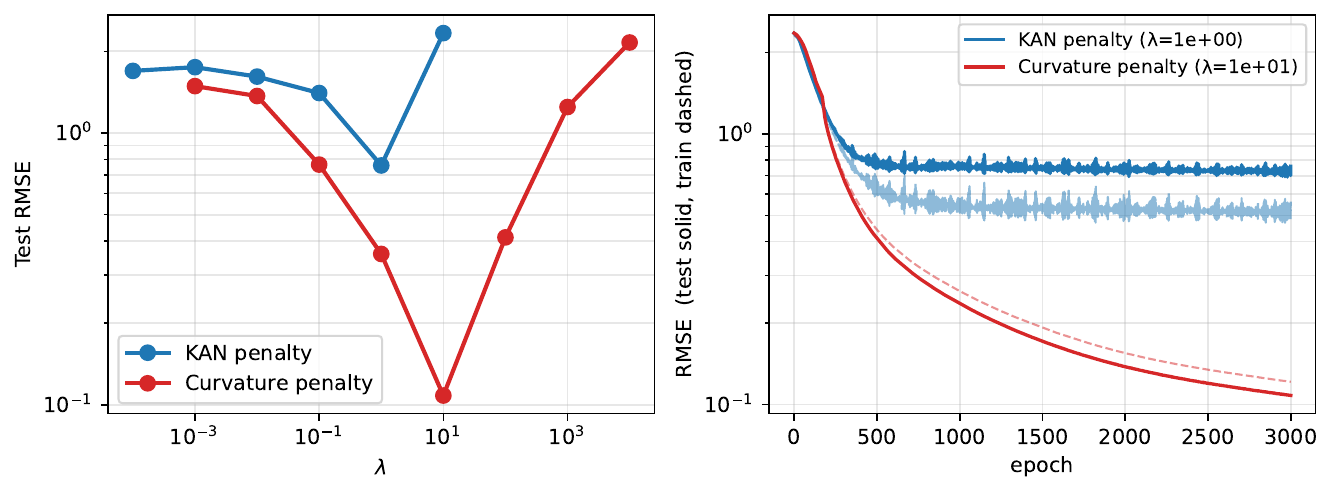}
\caption{
Over-parameterized KANs reach lower test RMSE under the curvature penalty than KAN, across all $\lambda$ values.
Penalty-vs-penalty comparison at high resolution $G=200$, architecture $\left[2, 1, 1\right]$, on the target $f(x, y) = \exp\left(\sin\left(\pi x\right) + y^{2}\right)$. \emph{Left}: final test RMSE as a function of the penalty coefficient $\lambda$. %
\emph{Right}: training trajectories at each penalty's best $\lambda$ (test RMSE solid, train RMSE dashed).
}
\label{fig:g200-comparison}
\end{center}
\end{figure*}

Finally, we study the effects of the penalty for different optimizers and model capacities.
Indeed, the preceding results all focus on Adam (default parameters, no weight decay)~\cite{kingma2014adam}. 
Figure~\ref{fig:adam-lbfgs-g200} shows that the curvature penalty also helps when fitting with L-BFGS~\cite{liu1989limited}, another popular optimizer that is often used to train KANs due to its use of curvature (of the loss) information.
The best combination---wider ($[2,5,1]$) architecture, L-BFGS, and curvature penalty---trains stably at fixed $G=200$ to test RMSE $\sim 10^{-3}$, where unpenalized L-BFGS at the same $G$ fails catastrophically (RMSE $\sim 0.5$).
Likewise, the curvature penalty helps the KAN scale its capacity correctly: going from narrow to wide architectures ($[2,1,1] \to [2,5,1]$) improves fit by $10\times$ (Adam) or $88\times$ (L-BFGS). 
The KAN penalty produces no such benefit.

\begin{figure*}[ht]
  \centering
  \includegraphics[width=0.85\linewidth]{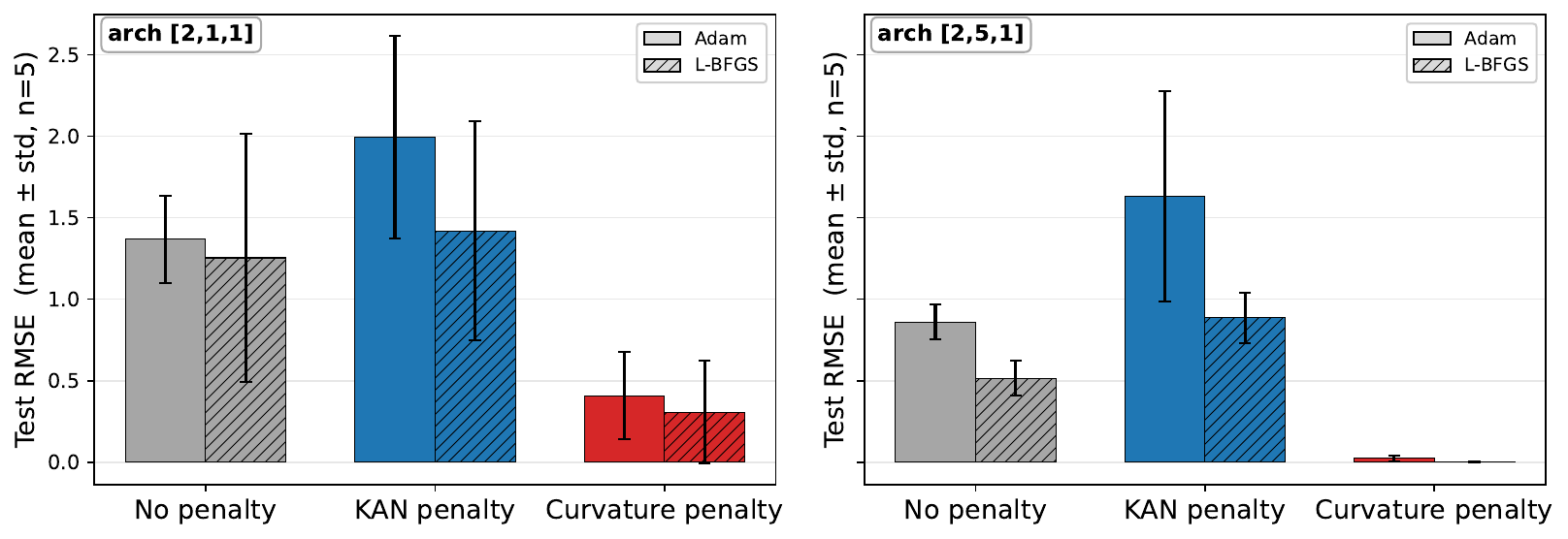}
  \caption{
  Curvature penalty helps multiple optimizers, improving the fit for $f(x, y) = \exp\!\big(\sin(\pi x) + y^2\big)$ at high grid resolution ($G=200$) without grid extension.}
  \label{fig:adam-lbfgs-g200}
\end{figure*}

\section{Compositional curvature}
\label{sec:compositional-curvature}

Despite the empirical effects shown above, an edge-wise curvature penalty (Eq.~\ref{eq:curv-penalty}), which is our focus in this work, is not a curvature penalty for a function composition.
Smoothness does not compose separably because the curvature of the composition depends on the curvatures of every layer through the chain rule, which introduces products of Jacobians.
In this section, we derive this relationship and
\begin{enumerate}[leftmargin=*]%
  \item observe that KAN layer Hessians are diagonal, yielding a simple per-edge chain-rule decomposition;
  \item show that the per-layer penalty is an upper bound on the composition's curvature norm, motivating its use;
  \item identify per-edge chain-rule weights that motivate a richer weighted penalty developed in Sec.~\ref{sec:richer-penalty}.
\end{enumerate}

\subsection{Warmup}
For two scalar functions in composition, $f(x) = \phi^{(2)}(\phi^{(1)}(x))$, the second derivative is
\begin{multline}
\label{eq:scalar-fdb}
f''(x) = \phi^{(2)\prime\prime}(\phi^{(1)}(x))  \left(\phi^{(1)\prime}(x)\right)^2 \\
+ \phi^{(2)\prime}(\phi^{(1)}(x))  \phi^{(1)\prime\prime}(x).
\end{multline}
The first term can give a significant amplification: a smooth $\phi^{(2)}$ applied to a steep (but smooth) $\phi^{(1)}$ can produce a wildly curved composition.
The edge-wise penalty misses this amplification, which is the key qualitative reason a (uniform) edge-wise penalty cannot be a tight curvature penalty on $f$.

\subsection{Tensor formulation}
Now let $\phi^{(1)} : \R^{n_0} \to \R^{n_1}$ and $\phi^{(2)} : \R^{n_1} \to \R^{n_2}$, and $f = \phi^{(2)} \circ \phi^{(1)}$. Use indices $i, j$ for inputs ($\R^{n_0}$), $b$ and $c$ for the intermediates ($\R^{n_1}$), and $a$ for outputs ($\R^{n_2}$). Write $z = \phi^{(1)}(x)$.
Twice differentiating the composition and applying the multivariate chain rule gives
\begin{multline}
\label{eq:two-layer-hessian}
\frac{\partial^2 f_a}{\partial x_i \partial x_j} =
\sum_{b,c} \left.\frac{\partial^2 \phi^{(2)}_a}{\partial z_b \partial z_c}\right|_{\phi^{(1)}(x)} \frac{\partial \phi^{(1)}_b}{\partial x_i} \frac{\partial \phi^{(1)}_c}{\partial x_j} \\
+ \sum_b \left.\frac{\partial \phi^{(2)}_a}{\partial z_b}\right|_{\phi^{(1)}(x)} \frac{\partial^2 \phi^{(1)}_b}{\partial x_i \partial x_j}.
\end{multline}
This (Fa\`a di Bruno for second-order) is the multivariate analog of Eq.~\ref{eq:scalar-fdb}.

Extending to multiple layers, we first write the layerwise Jacobian $J_\ell \in \R^{n_\ell \times n_{\ell-1}}$ and Hessian $H_\ell \in \R^{n_\ell \times n_{\ell - 1} \times n_{\ell - 1}}$:
\begin{align*}%
(J_\ell)_{ab} &=
\frac{\partial \phi^{(\ell)}_a}{\partial z_{\ell - 1, b}}, \\ 
(H_\ell)_{abc} &= \frac{\partial^2 \phi^{(\ell)}_a}{\partial z_{\ell - 1, b} \, \partial
z_{\ell - 1, c}}.
\end{align*}
Then Eq.~\ref{eq:two-layer-hessian} becomes the matrix Hessian
\begin{equation}
\label{eq:two-layer-hessian-matrix}
\nabla^2 f_a = J_1^\top H_{2a} J_1 + \sum_c (J_2)_{ac} H_{1c}.
\end{equation}
(Note that 
both $J_\ell$ and $H_\ell$ are functions of the layer's input $z_{\ell - 1}$, with $z_0 = x$; for brevity we suppress evaluation arguments in the multi-layer expressions.)
To iterate beyond the first two layers, introduce the upstream and downstream Jacobian products:
\begin{align}
U_\ell &= J_L J_{L-1} \cdots J_{\ell + 1}, \quad U_L = I_{n_L}, \\
D_\ell &= J_{\ell - 1} J_{\ell - 2} \cdots J_1, \quad D_1 = I_{n_0}.
\end{align}
Then by induction on $L$, we have
\begin{equation}
\label{eq:depth-L}
 \nabla^2 f_a \;=\; \sum_{\ell=1}^{L} \sum_c (U_\ell)_{ac}\, D_\ell^\top H_{\ell c}\, D_\ell.
 \end{equation}
Each summand contains exactly one Hessian factor; products of Hessians from different layers only appear at derivatives of order $\geq 3$.

\subsection{KAN's layer Hessians are diagonal}
\label{subsec:kan-hessians-diagonal}

A general vector-valued $\phi^{(\ell)}$ has a full Hessian tensor with $n_\ell n_{\ell - 1}^2$ entries. KAN layers, however, take the form
\begin{equation}
\phi^{(\ell)}_a(z) = \sum_b \phi^{(\ell)}_{ab}(z_b),
\end{equation}
where each activation function $\phi^{(\ell)}_{ab}$ is univariate. Consequently, all mixed second partials vanish:
\begin{equation}
(H_\ell)_{abc} = \delta_{bc} \, \phi^{(\ell)\prime\prime}_{ab}(z_b).
\end{equation}
The Hessian is \emph{diagonal} in its two input indices.

Plugging into Eq.~\ref{eq:depth-L} and collapsing the input pair via the delta, and letting $\Phi^{\prime\prime}_{\ell c}$ denote the diagonal matrix with entries $(\Phi^{\prime\prime}_{\ell c})_{bb} = \phi^{(\ell)\prime\prime}_{cb}(z_{\ell - 1, b})$, gives
\begin{equation}
\label{eq:kan-hessian}
\nabla^2 f_a \;=\; \sum_{\ell=1}^{L} \sum_c (U_\ell)_{ac}\, D_\ell^\top \Phi^{\prime\prime}_{\ell c}\, D_\ell.
\end{equation}
This is a substantial simplification over the analogous expansion for an MLP (where the post-activation Hessian through linear layers does not have this diagonal structure). Each term is a product of three transparent factors:
\begin{itemize}[leftmargin=*]
\item $\phi^{(\ell)\prime\prime}_{cb}(z_{\ell - 1, b})$: the pointwise curvature of edge $b \to c$ at layer $\ell$, evaluated at the activation value $z_{\ell - 1, b}$ that the edge sees during the forward pass.
\item $(U_\ell)_{ac}$: the upstream sensitivity from node $c$ at layer $\ell$ to output $a$; a sum over forward paths.
\item $(D_\ell)_{bi} (D_\ell)_{bj}$: the downstream sensitivities from input directions $i$ and $j$ to node $b$ at layer $\ell - 1$; a sum over backward paths, taken twice.
\end{itemize}
We argue that this clean structure is itself a noteworthy structural advantage of KANs over MLPs for principled smoothness analysis: the chain-rule decomposition stays interpretable (see Sec.~\ref{sec:discussion}).

\subsection{The edge-wise penalty is an upper bound}
\label{sec:cs-bound}

Averaging the squared input-Hessian norm over the training distribution gives the composition-level curvature of the function we are fitting:
\begin{equation}
\label{eq:r-comp}
\Rcomp(f) := \E_{x \sim p_{\text{data}}} \norm{\nabla^2 f(x)}_F^2.
\end{equation}
As a stand-alone penalty, however, $\Rcomp$ has two drawbacks: it is coupled to the data distribution, evaluating curvature only at inputs the model has seen and leaving the function unconstrained off-data; and it competes with the loss wherever the target itself has curvature, risking over-restriction where expressivity is needed. 
Even so, $\Rcomp$ can be evaluated directly by autograd-through-Hessian or estimated cheaply with a Hutchinson trace estimator~\cite{hutchinson1990stochastic}---though both require a double backward. 
We instead derive a coefficient-space upper bound from Eq.~\ref{eq:kan-hessian} that decouples the prior from the empirical input measure while preserving the chain-rule structure that ties hidden-layer curvature to the composition Hessian.

Substituting the indexed form of Eq.~\ref{eq:kan-hessian} into $\norm{\nabla^2 f}_F^2 = \sum_{a, i, j} \left((\nabla^2 f)_{aij}\right)^2$ and applying Cauchy--Schwarz over the $E = \sum_\ell n_\ell n_{\ell-1}$ edges, $\left(\sum_e t_e\right)^2 \le E \sum_e t_e^2$, gives
\begin{equation}
\label{eq:cs-bound}
\Rcomp(f) \;\le\; E \sum_{\ell=1}^{L} \sum_{b, c} \E_x \!\left[\phi^{(\ell)\prime\prime}_{cb}(z_{\ell-1, b})^2 \, w^{(\ell)}_{cb}(x)\right],
\end{equation}
where the data-dependent \textit{path weight} is
\begin{equation}
\label{eq:edge-weight}
w^{(\ell)}_{cb}(x) = \sum_a (U_\ell)_{ac}^2 \sum_{i, j} (D_\ell)_{bi}^2 (D_\ell)_{bj}^2.
\end{equation}
Each edge $b \to c$ at layer $\ell$ contributes its squared curvature, weighted by the squared upstream sensitivity to the outputs and the squared downstream sensitivities to the inputs.

The right-hand side of Eq.~\ref{eq:cs-bound} reduces to a coefficient-space penalty under three structural assumptions. For each edge $e$, let $\Omega_e \subset \R$ be the grid range.
\begin{itemize}[leftmargin=*]%
\item \textbf{(A1)} \emph{Bounded coefficient of variation of the path weight:} $\sigma_{w_e}/\bar w_e \le \kappa$ for every edge $e$, where $\bar w_e := \E_x[w_e(x)]$.
\item \textbf{(A2)} \emph{Bounded density on each spline support:} the density of activation inputs reaching edge $e$ under $x \sim p_{\text{data}}$ is bounded, $p_e(z) \le C/|\Omega_e|$ pointwise on $\Omega_e$, with $C \ge 1$.
\item \textbf{(A3)} \emph{Knot spacing bounded by 1:} $h_e := |\Omega_e|/G \le 1$ for every edge $e$.
\end{itemize}
A1 and A2 are properties of the chain Jacobians and the input distribution, not of training, and both are post-hoc verifiable per edge. A3 is a mild architectural condition---it holds whenever the grid resolution is at least as fine as the input domain extent, which is the regime of any practically-trained KAN.

We now prove this section's main result:

\begin{theorem}[Edge-wise penalty bounds composition curvature]
\label{thm:cs-bound}
Suppose (A1)--(A3). Then
\begin{equation}
\label{eq:perEdge-upperbound}
\Rcomp(f) \;\le\; K_\lambda \, R(f),
\end{equation}
where
\begin{gather*}
K_\lambda := 2 E C \left(\max_e \gamma_e\right)\left(\max_e \bar w_e\right) \, \frac{G^3}{(\min_e |\Omega_e|)^4}, \\
\gamma_e := 1 + \kappa\, \sigma_x({\phi_e''}^2)/\E_x[{\phi_e''}^2].
\end{gather*}
\end{theorem}
\begin{proof}
Decompose $\E_x[w_e \phi_e''^2] = \bar w_e\, \E_x[\phi_e''^2] + \mathrm{Cov}_x(w_e, \phi_e''^2)$.
Bound the covariance by Cauchy--Schwarz under (A1). 
Lift the data expectation to a function-space integral under (A2).
Separate SiLU and spline terms with Young's inequality.
Bound the basis curvature Gram via Gershgorin and the partition-of-unity identity ($M_e \preceq h_e^{-3} D_2^\top D_2$), then apply (A3) to absorb $h_e^{-3}$. 
Together:
\begin{align*}
\Rcomp(f)
&\overset{\mathrm{(A1)}}{\le} E \sum_e \bar w_e\, \E_x[\phi_e''(z_e)^2]\, \gamma_e \\
&\overset{\mathrm{(A2)}}{\le} E \sum_e \frac{C\,\gamma_e\,\bar w_e}{|\Omega_e|}\!\int_{\Omega_e}\!\phi_e''(z)^2\, dz \\
&\;\,\le \;\,E \sum_e \frac{2 C\,\gamma_e\, \bar w_e}{|\Omega_e|}\!\big[\ws[e]^2\, c_e^\top M_e\, c_e + K_{\mathrm{silu}}\, \wb[e]^2\big] \\
&\overset{\mathrm{(A3)}}{\le} K_\lambda \sum_e \big[\norm{D_2(\ws[e]\, c_e)}^2 + K_{\mathrm{silu}}\, \wb[e]^2\big] \\
&\;\,=\;\, K_\lambda \, R(f).
\end{align*}
\end{proof}
Theorem~\ref{thm:cs-bound} establishes that the edge-wise penalty $R(f)$ from Sec.~\ref{sec:activation-curvature} upper-bounds the composition curvature $\Rcomp(f)$ up to a multiplicative constant.  
Operationally, $K_\lambda$ folds into the regularization weight $\lambda$ at training time, so minimizing $R(f)$ minimizes a rigorous upper bound on $\Rcomp(f)$.

\paragraph{Remark.}
Ideally, a function-space smoothness prior should not depend on the training distribution. 
Equation~\ref{eq:r-comp} does, in three ways exposed by the proof of Theorem~\ref{thm:cs-bound}: a joint per-sample product $\phi_e''^2 \, w_e$, a marginal density of $z_e$ on $\Omega_e$, and a marginal per-edge importance $\bar w_e$. 
The first lets the optimizer game the penalty by anti-correlating $w_e$ and $\phi_e''^2$ pointwise rather than reducing curvature. 
The second makes the prior a property of the function's behavior \emph{on the data} rather than its shape on $\Omega_e$, so the activation can be arbitrarily rough on unvisited regions of $\Omega_e$. 
The third weights each edge's curvature contribution by its average path weight under the data, making the prior's per-edge balance a function of $p_{\text{data}}$; it drops via $\bar w_e \le \max_e \bar w_e$. 
This final coupling is the mildest of the three, only a per-edge re-weighting rather than a gameable joint or a blind density. Does retaining $\bar w_e$ inside the sum, which targets the edges whose curvature most impacts the composition, give a tighter, more informative penalty than $R(f)$ in practice? We study this empirically in Sec.~\ref{sec:richer-penalty}.

\section{A richer curvature penalty}
\label{sec:richer-penalty}

The curvature penalty (Eq.~\ref{eq:curv-penalty}), by working edge-wise, ignores compositionality.
Our derivation of the upper bound reveals a simple affordance for incorporating some of the lost information:
retain the $\bar w_{e=(\ell, b, c)} = \E_{x}\!\big[w^{(\ell)}_{cb}(x)\big]$ terms and use them to compute a weighted edge-wise penalty, in which each edge's contribution scales with its expected impact on the composition Hessian via the chain rule.
The one downside to this approach is that it reintroduces the data into the penalty, as the weighting term takes the expectation over the training data. 
But if that is acceptable, the diagonal structure of KAN Hessians (Sec.~\ref{subsec:kan-hessians-diagonal}) makes this attractive:
\begin{gather}
w^{(\ell)}_{cb}(x) = \norm{D_{\ell, b, :}(x)}^4 \, \norm{U_{\ell, :, c}(x)}^2, \label{eq:w-e-explicit} \\
R_{\bar w}(f) = \sum_e \bar w_e \left( \norm{D_2(\ws[e]\, c_e)}^2 + K_{\mathrm{silu}}\, \wb[e]^2 \right). \label{eq:weighted-penalty}
\end{gather}

To test this richer penalty, we do a 10-seed comparison between weighted and uniform edge-wise penalties, first selecting the best $\lambda$ for each method as the penalty scales will differ.
As Fig.~\ref{fig:weighted-vs-uniform} shows, the weighted penalty outperforms the uniform penalty in 9 of 10 seeds, yielding a $2.2\times$ lower mean test RMSE ($0.0119 \pm 0.0086$ vs.\ $0.0260 \pm 0.0172$; paired Wilcoxon $p = 0.014$).
\begin{figure}[ht]
  \centering
  \makeatletter
  \includegraphics[width=\if@twocolumn 0.9\else 0.65\fi\linewidth]{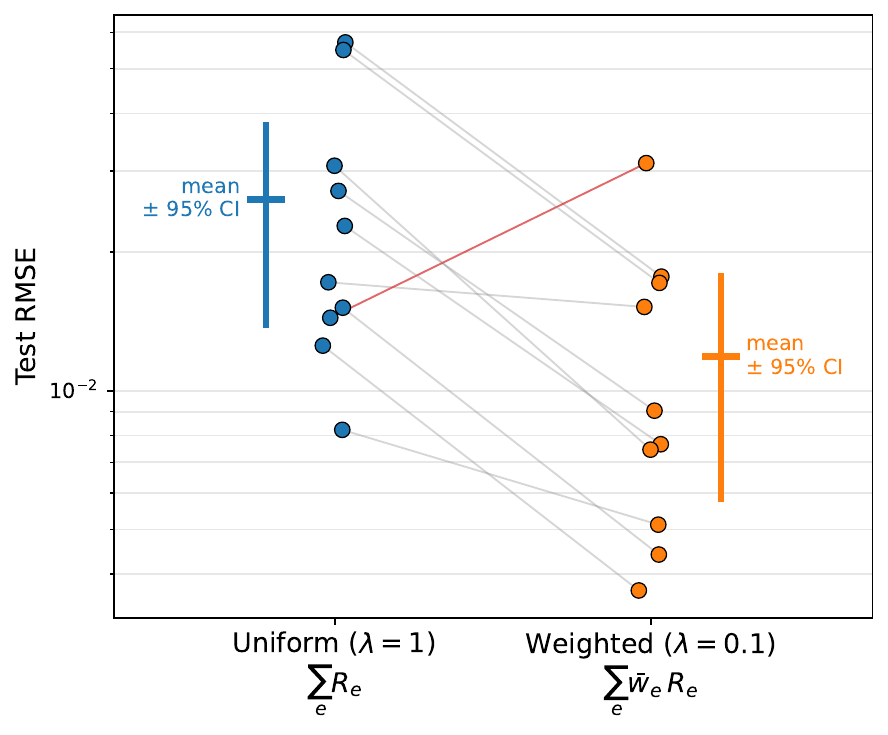}
  \makeatother
  \caption{Weighted curvature penalty more than halves the mean test RMSE of uniform, $n=10$ seeds at $[2,5,1]$, $G=200$ on target $f(x,y) = \exp(\sin(\pi x) + y^2)$.
  The value of $\lambda$ for each penalty was identified by a single-seed sweep over $\lambda \in \{10^{-3}, 10^{-2}, 10^{-1}, 1, 10, 10^{2}, 10^{3}\}$ followed by a fill at the minimum. 
  CI bars at the sides.}
  \label{fig:weighted-vs-uniform}
\end{figure}

\section{Related work}
\label{sec:related}

\paragraph{Penalized splines.} The Eilers--Marx P-spline framework~\cite{eilers1996flexible, eilers2010splines, eilers2015twenty} has 30 years of development including REML-based smoothing parameter selection, mixed-model duality, GAMs with P-splines, tensor-product P-splines for multidimensional smoothing, Bayesian P-splines, and shape-constrained variants. None of this work concerns nested or compositional spline structures.

\paragraph{KAN curvature.}
Two recent papers reach for second-derivative penalties on KANs without connecting to the penalized-spline literature. Zheng et al.~\cite{zheng2025frkan} state an objective involving second derivatives of spline coefficients (their eq.~8), but as published the expression is a signed sum without integration or squaring and is not a curvature penalty; we were also unable to match it to their released code. 
Cang et al.~\cite{yu2024can}, in contrast, also pursue an $L^2$ curvature penalty. 
However, their formulation penalizes curvature of the spline component only, not $\phi$, and they do not connect to compositional curvature. 
This leaves the SiLU contribution unconstrained, so $\alpha_e$ can absorb arbitrary curvature over the composition with no penalty cost and the quantity does not bound $\Rcomp$. 
Our derivation makes these issues explicit, and the form $\norm{D_2(\ws[e]\, c_e)}^2 + K_{\mathrm{silu}}\, \wb[e]^2$ closes the gap by penalizing the full activation. 
Rigas \textit{et al.}~\cite{rigas2026idf} introduce an adaptive grid method that places spline knots at high-curvature locations, but do not propose a penalty.
To the best of our knowledge, no prior KAN work cites Eilers and Marx or draws from the broader penalized-spline literature.

\paragraph{Deep-net first-order smoothness.} Drucker and LeCun's ``double backpropagation''~\cite{drucker1992double} introduced an input-Jacobian penalty. 
Subsequent work includes contractive autoencoders~\cite{rifai2011contractive}, Sobolev training~\cite{czarnecki2017sobolev}, and a large literature on Lipschitz constraints~\cite{anil2019sorting, miyato2018spectral, virmaux2018lipschitz, gouk2021regularisation}. 
The Lipschitz-of-a-composition-as-product-of-layer-Lipschitz idea is the first-order analog of our second-order chain-rule decomposition. 
Recent work~\cite{liu2022lipschitz,mcginnis2026beyond} explores depth-dependent or non-uniform layer-wise Lipschitz allocation.

\paragraph{Input-Hessian penalty: function vs.\ loss curvature.} Curvature penalties based on input-Hessians have been pursued for objectives quite different from ours, and on a different mathematical object than the one our framework targets. Both CURE~\cite{moosavi2019robustness} and CRR~\cite{curvaturerate2025} penalize the input-Hessian of the \emph{loss}, $\nabla_x^2 \mathcal{L}(f(x), y)$: CURE via finite-difference Hessian-vector products in the gradient direction --- motivated by adversarial robustness, since the worst-case perturbation $\delta$ maximizes the second-order Taylor term $\delta^\top \nabla_x^2 \mathcal{L}\, \delta$ --- and CRR via a curvature-rate measure of the loss landscape for sharpness-aware learning. 

We instead penalize the input-Hessian of the \emph{function output}, $\nabla_x^2 f$, which controls the smoothness of the learned function itself, not the response of the loss to input perturbation. The two are mathematically distinct: for regression with mean-squared error, $\nabla_x^2 \mathcal{L} = 2(f - y)\, \nabla_x^2 f + 2\, \nabla_x f\, \nabla_x f^\top$ contains both a function-Hessian piece and an input-Jacobian outer product, and at a perfect fit ($f = y$) reduces to pure first-order sensitivity --- carrying no second-order information about $f$. The objectives also differ: CURE/CRR want $f$ insensitive to input perturbations regardless of what the data demands; our setting wants $f$ to fit smooth structure while preserving the chain rule's role in composing roughness from smooth atomic edges. None of these works derives a chain-rule decomposition of $\nabla_x^2 f$ to motivate per-layer penalty design for compositional architectures.

\paragraph{Loss-Hessian on parameters (sharpness-aware methods).} A third, distinct object is the Hessian of the loss with respect to the \emph{parameters} $\theta$. Hessian-trace penalties~\cite{liu2023hessian, crsam2024} penalize $\nabla_\theta^2 \mathcal{L}$ for sharpness-aware generalization (flat minima as a generalization proxy). This serves a separate purpose from the input-derivative penalties above and should not be confused with them.

\section{Discussion}
\label{sec:discussion}

Do we want smooth activations when training KANs?
If all we care about is predictive accuracy, it doesn't matter what the inner functions look like so long as their composition fits the data.
Vitushkin's theorem~\cite{vitushkin1964proof} in fact shows that smooth representations need not exist; Samadi \textit{et al.}~\cite{samadi2024smooth} discuss its implications for KANs. 
The univariate inner functions are generically non-smooth even for a smooth target function.
Smoothness, then, is a modeling choice.
Visualizing the activations makes KANs more interpretable than alternatives such as MLPs, and we argue that smooth activations (or at least activations as smooth as possible to fit the data) are the right prior for maximally interpretable KANs.
Interpretability may be valuable enough to justify a small accuracy cost.
KANs open the black box of deep neural networks, and we should press this advantage.

The curvature penalty stabilizes overparameterized KAN training at fixed high $G$, without grid extension. A practical consequence is that KAN optimization remains a single end-to-end differentiable run, rather than a multi-stage curriculum with discrete grid changes and optimizer reinitializations.
This matters whenever KANs compose with broader systems, such as joint training with other components, neural architecture search, and meta-learning, all of which assume a single optimization problem with stable gradient flow.
Online or streaming settings, where one cannot revisit an early low-$G$ phase, likewise rule out grid extension.
While multi-stage grid extension remains effective, the curvature penalty makes single-stage training a viable alternative.

A single structural property of KANs underlies both interpretability and our smoothness analysis: layer outputs are sums of univariate edges, which makes layer Hessians diagonal. 
This supports both per-edge inspectability (the visualization-based interpretability that motivates KANs) and per-edge attribution of compositional curvature (the principled smoothness penalty we develop here). 
MLPs have neither property: their dense layers entangle weights and nonlinearities so that curvature emerges from interactions that make attribution difficult.
Smoothness is not a feature added to KANs but a property their architecture invites us to control.

Several limitations of our results merit future work.
The focus on the edge-wise curvature penalty means variants should be explored more fully, beyond the initial work given in Sec.~\ref{sec:richer-penalty}. 
In particular, it remains open how much of the compositional structure should be folded into the penalty.
Another open issue is the treatment of grid spacing and grid updates. 
We focused on a fixed, uniform grid of knots (constant $h_e$), yet variable grid spacing, which is needed for techniques such as adaptive grids, requires a penalty that accounts for non-uniform $h_e$.
The penalized-splines literature has also considered penalty norms besides the $L^{2}$ form we use.
Eilers \emph{et al.} note~\cite{eilers2015twenty} that an $L^{1}$ penalty on second differences can promote piecewise-linear splines over globally smooth ones. 
Given the success of ReLU, this kink-tolerant variant may prove effective in KANs.
On the theoretical side, our treatment leaves room for improvement: tighter bounds on the compositional curvature, for instance, may point toward novel penalties.

In summary, the curvature penalty gives KAN training a single, principled smoothness lever that promotes interpretability, stabilizes high-resolution fitting, and respects the network's compositional structure. 
Smoothness-aware KAN design is now a tractable problem rather than an aspiration.
For an architecture whose appeal rests on interpretability, such advancements may further cement KANs as a credible foundation for interpretable scientific machine learning.

{\footnotesize
\bibliographystyle{IEEEtran}
\bibliography{refs}
}

\appendix

\section{Methods}
\label{sec:methods}

All experiments use PyTorch v2.5.1~\cite{paszke2019pytorch} on CPU. %
FastKAN ablation in App.~\ref{sec:beyond-bsplines} uses the original implementation~\cite{li2024fastKAN}.
Every B-spline KAN edge uses cubic ($k=3$) splines on a fixed, uniform knot grid of size $G$, with a SiLU base. %
Each experiment draws $n_{\text{train}} = 1024$ training and $n_{\text{test}} \in \{256, 1024\}$ test inputs uniformly from the target's domain $\Omega \subset \R^d$, and computes targets in closed form. New samples are drawn for each random seed.
Where reported, best $\lambda$ is selected by a single-seed scout sweep over a range adjusted per penalty type, then re-evaluated with multiple seeds at the winner.
Total curvature $\sum_e \int_{\Omega_e} \phi_e''(z)^2\, dz$ is evaluated by trapezoidal integration on a dense grid ($n=4096$) covering each edge's spline support, with $\phi_e''$ computed via the analytic B-spline second-derivative identity for the spline term and the closed-form SiLU second derivative for the base.

\paragraph{Optimization.}
Adam: %
$\text{lr}=10^{-3}$, no weight decay, $(\beta_1, \beta_2)=(0.9, 0.999)$, $\epsilon=10^{-8}$. Trained for 3000 epochs with a 200-epoch warmup phase (no penalty during warmup), batch size 256 (64 for the smaller $[2,2,1,1]$ anecdote setup). %
L-BFGS: $\text{lr}=1.0$, $\text{max\_iter}=20$ inner iterations per outer step, $\text{history\_size}=100$, strong-Wolfe line search, $\text{tolerance\_grad}=10^{-9}$, $\text{tolerance\_change}=10^{-12}$. 
Trained for 500 outer steps without warmup on the full training batch. %

\section{Beyond B-splines}
\label{sec:beyond-bsplines}

\begin{figure*}
\begin{center}
\includegraphics[width=\textwidth]{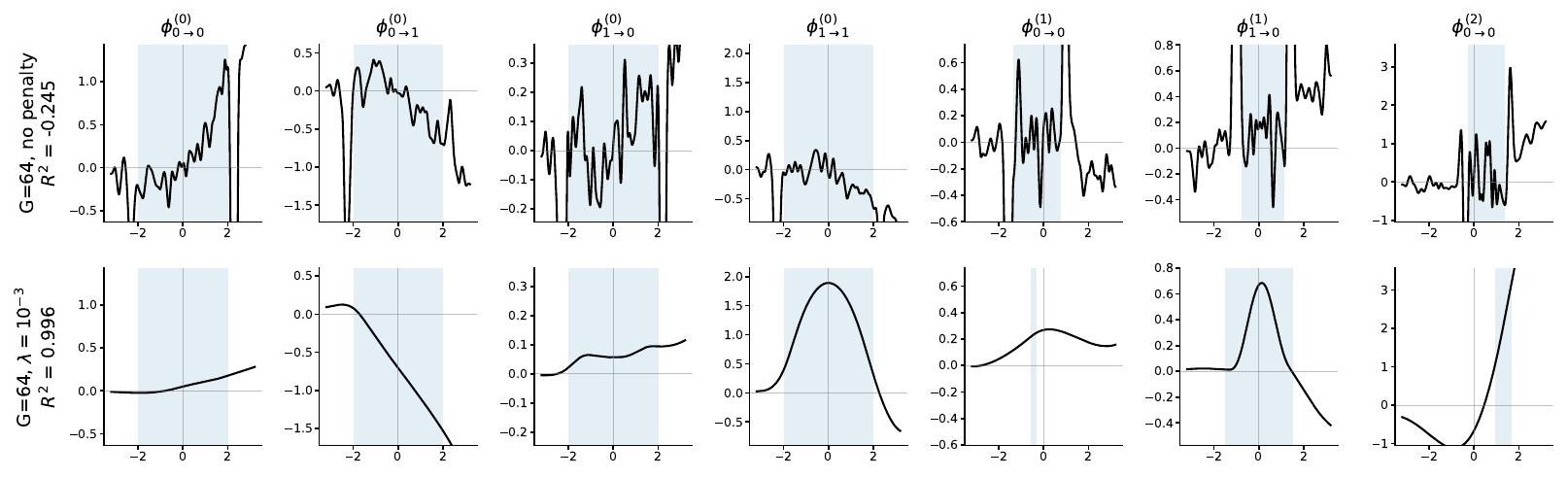}
\caption{Activation functions for high resolution FastKAN trained on $f(x, y) = \sin\left(x + y^{2}\right)$ over $\left[-2, 2\right]^{2}$ with architecture $\left[2, 2, 1, 1\right]$ and Gaussian-RBF basis size $G=64$. \emph{Top}: no penalty, $R^2 = -0.24$. 
\emph{Bottom}: same configuration with our curvature penalty (Eq.~\ref{eq:fastkan-penalty}) at $\lambda = 10^{-3}$, $R^2 = 0.996$. 
The penalty makes the overparameterized regime trainable, exactly as in the B-spline setting.
Shaded regions denote the range of activation inputs.
}
\label{fig:fastkan}
\end{center}
\end{figure*}

The curvature penalty is basis-agnostic.
To see an example beyond B-splines, here we briefly consider an RBF basis as implemented by FastKAN~\cite{li2024fastKAN}.
FastKAN uses a Gaussian radial basis function with uniform centers,
\begin{equation}
\label{eq:fastkan-basis}
B_g(x) = \exp\!\left(-\left(\frac{x - \mu_g}{h}\right)^2\right), \;\;
\mu_g = \mu_0 + g h,
\end{equation}
with bandwidth $h$ equal to the center spacing. 
Each activation has the form
\begin{equation}
\label{eq:fastkan-edge}
\phi(x) \;=\; \sum_g c_g B_g(x) \;+\; \wb\, \mathrm{SiLU}(x),
\end{equation}
with no separate spline scale.
For a Gaussian RBF with equispaced centers the curvature Gram matrix has closed form:
  \begin{equation}
    M_{ij} \;=\; \frac{\sqrt{\pi/2}}{h^3}\, e^{-\Delta^2/2}\, \left(\Delta^4 - 6\Delta^2 + 3\right),
  \end{equation}
where $\Delta = j - i$. 
The diagonal entries are $M_{ii} = 3\sqrt{\pi/2}/h^3$ while the off-diagonal entries decay as a Gaussian in $\Delta$, making $M$ effectively banded.

For the full edge function (Eq.~\ref{eq:fastkan-edge}), we have curvature
\begin{equation}
\label{eq:fastkan-full}
\int \left(\phi''\right)^2 dx \;=\; c^\top M\, c \;+\; 2\, \wb\, c^\top k \;+\; \wb^2\, K_{\mathrm{silu}},
\end{equation}
where $k$ has entries $k_g = \int B_g''(x)\, \mathrm{SiLU}''(x)\, dx$ and 
$K_{\mathrm{silu}} \approx 0.443$ as before. 
Here $k$ admits no fully elementary closed form but can be precomputed numerically. 
However, as before, we drop the cross-terms. ($\mathrm{SiLU}''$ is a bump localized near zero, so $\norm{k}$ is small except for those $\mu_g$ that overlap the bump, and the cross term is bounded by $2\left|\wb\right| \norm{c} \norm{k}$.)
The penalty for FastKAN comes from summing over all edges $(i, j)$ in all layers:
\begin{equation}
\label{eq:fastkan-penalty}
R_{\text{FK}}
= \sum_\ell \sum_{i, j} \left[ \left(c^{(\ell)}_{ij}\right)^\top M\, c^{(\ell)}_{ij} \;+\; K_{\mathrm{silu}}\, \left(\alpha^{(\ell)}_{ij}\right)^2 \right],
\end{equation}
 where $c_{ij}^{(\ell)} \in \R^G$ is the basis-coefficient vector for edge $(i,j)$ in layer $\ell$.

We applied our penalty to a FastKAN\footnote{We disable FastKAN's optional LayerNorm in both conditions: it is incompatible with the singleton $1 \to 1$ layer in $\left[2, 2, 1, 1\right]$, and its learnable affine introduces a gauge symmetry that our edge-wise curvature analysis does not account for.} training run with high resolution $G=64$ on the target $f(x, y) = \sin\left(x + y^{2}\right)$ over $\left[-2, 2\right]^{2}$ with architecture $\left[2, 2, 1, 1\right]$, and visualized the activation functions in Fig.~\ref{fig:fastkan}.
The penalized model had notably smoother activation functions and was substantially more accurate ($R^2 = 0.996$ vs.\ $-0.24$).
Figure~\ref{fig:fastkan}, together with our prior results, demonstrates that the curvature penalty is effective across multiple KAN bases.

\end{document}